\documentclass[conference]{IEEEtran}

\usepackage{cite}
\usepackage{url}
\usepackage{amsmath,amssymb,amsfonts,amsthm}
\usepackage{latexsym}
\usepackage{booktabs}
\usepackage{enumitem}
\usepackage{graphicx}
\usepackage{color}
\usepackage{bm}
\usepackage{algorithm}
\usepackage[noend]{algorithmic}
\usepackage{mathtools}
\usepackage{textcomp}
\usepackage{xcolor}

\newtheorem{theorem}{Theorem}
\newtheorem{lemma}{Lemma}

\newtheorem{remark}{Remark}

\begin{document}

\title{Robust Multi-Agent Bandits with Heavy-Tailed Rewards and Information Asymmetry}

\author{\IEEEauthorblockN{Daphne Feng\IEEEauthorrefmark{1},
Ricardo Parada\IEEEauthorrefmark{2},
Lily Jiang\IEEEauthorrefmark{1},
Sophia Yi\IEEEauthorrefmark{1},
William Chang\IEEEauthorrefmark{1}}
\IEEEauthorblockA{\IEEEauthorrefmark{1}Department of Mathematics, University of California, Los Angeles, Los Angeles, CA, USA\\
\IEEEauthorrefmark{2}Department of Mathematics, University of California, Riverside, Riverside, CA, USA}
}

\maketitle

\begin{abstract}
The multi-armed bandit problem is a central framework in sequential decision-making, extensively studied under sub-Gaussian reward assumptions. However, real-world applications often involve heavy-tailed reward distributions and decentralized, information-asymmetric interactions. We study multi-agent multi-armed bandits with heavy-tailed rewards under three information-asymmetry regimes: unobserved actions with common rewards, observed actions with independent rewards, and unobserved actions with independent rewards. We develop robust decentralized algorithms for each setting and derive regret guarantees that nearly match centralized heavy-tailed rates. Experiments on a Pareto-distributed reward environment validate our theoretical findings and illustrate the trade-offs between synchronization, coordination, and exploration across the three regimes.
\end{abstract}



\section{Introduction}

The multi-armed bandit (MAB) problem is a core model for sequential decision-making under uncertainty, originating in work on adaptive experimentation and Bayesian selection~\cite{Robbins1952SomeAO, thompson1933likelihood}. At each round a learner selects an action and observes a random payoff, balancing exploration of uncertain actions against exploitation of apparently good ones. Bandit models underpin data-driven decision systems---online experimentation, recommendation, resource allocation, spectrum access, multi-robot coordination---which are typically \emph{distributed} in ways not captured by single-agent abstractions: several agents learn simultaneously while each observes only part of the system state.

The multi-player MAB (MMAB) literature spans several information structures. In one line, players share information over communication graphs or gossip protocols~\cite{kleinberg2008collaborative, szorenyi2013gossip}. Another has players choose from a common arm set where collisions couple outcomes~\cite{kalathil2014decentralized, shi2021collision}. More recently, cooperative MMAB has been studied under limited or no communication with structured observation asymmetries~\cite{chang2021multiplayer, chang2023optimal}; see~\cite{boursier2024survey} for a survey. These works show that even without explicit messaging, agents can sometimes coordinate through shared structure or pre-agreed protocols.

A largely orthogonal challenge is that many reward signals are \emph{heavy-tailed}: rare extreme events dominate observations, producing weak concentration and rendering sub-Gaussian analyses inaccurate. Heavy tails arise naturally in financial returns, network traffic bursts, and outlier-prone performance metrics. Robust algorithms for heavy-tailed rewards include robust-UCB methods~\cite{bubeck2012banditsheavytail} and deterministic exploration--exploitation schedules~\cite{vakili2013dsee}, with extensions to pure exploration~\cite{yu2018pure}, linear bandits~\cite{shao2018almost}, and minimax-optimal procedures~\cite{lee2024minimax}, while Catoni-style confidence sequences sharpen what is achievable under weak moment assumptions~\cite{wang2023catoniCS}.

The \emph{intersection} of cooperative multi-agent bandits, heavy-tailed rewards, and decentralized operation with no online communication remains underexplored. Existing multi-agent heavy-tailed work relies on explicit communication: \cite{dubey2020cooperative} considers delayed message passing and~\cite{wang2025mmabfully} studies graph-based communication. We ask: what is achievable when agents coordinate implicitly via a pre-agreed protocol?

\textbf{Our contributions.} We introduce three problem formulations capturing distinct information asymmetries in multi-agent heavy-tailed bandits: common rewards with unobserved actions (Problem~A), independent rewards with observed actions (Problem~B), and independent rewards with unobserved actions (Problem~C). For each we develop a robust decentralized algorithm---\texttt{mRUCB-A}, \texttt{mRUCB-Intervals}, and \texttt{mHT-DSEE}---and prove regret guarantees summarized in Table~\ref{tab:summary}. The robust mean estimator and the single-agent concentration arguments are adapted from~\cite{bubeck2012banditsheavytail, vakili2013dsee}; our contribution lies in the multi-agent formulation, in the use of intentional action deviations as an implicit signaling channel, and in a unified comparison of what each information structure costs.

\begin{table}[t]
\caption{Summary of information structures and regret bounds.}
\label{tab:summary}
\centering
\begin{tabular}{lccc}
\toprule
& \textbf{Prob.\ A} & \textbf{Prob.\ B} & \textbf{Prob.\ C} \\
\midrule
Actions observed? & No & Yes & No \\
Rewards shared? & Yes & No (i.i.d.) & No (i.i.d.) \\
Algorithm & \texttt{mRUCB-A} & \texttt{mRUCB-Int.} & \texttt{mHT-DSEE} \\
Regret & $O\!\left(\!\frac{\log T}{\Delta^{1/\varepsilon}}\!\right)$ & $O\!\left(\!\frac{\log T}{\Delta^{1/\varepsilon}}\!\right)$ & $O(\log^2 T)$ \\
\bottomrule
\end{tabular}
\end{table}


\section{Preliminaries}\label{sec:preliminary}

\subsection{Heavy-tailed bandits}

Consider a stochastic MAB with $K$ arms. Each arm $\bm{a}\in\mathcal{A}:=\{1,\dots,K\}$ has an unknown reward distribution $\nu_{\bm{a}}$ with mean $\mu_{\bm{a}}$. At round $t$, the agent selects arm $\bm{a}_t$ and observes a reward drawn from $\nu_{\bm{a}_t}$. The expected regret at horizon $T$ is
\begin{equation}\label{eq:regret_single}
R_T = T\mu^\star - \sum_{t=1}^T \mathbb{E}[\mu_{\bm{a}_t}] = \sum_{\bm{a}\in\mathcal{A}} \Delta_{\bm{a}}\,\mathbb{E}[n_{\bm{a}}(T)],
\end{equation}
where $\mu^\star = \max_{\bm{a}} \mu_{\bm{a}}$, $\Delta_{\bm{a}} = \mu^\star - \mu_{\bm{a}}$ is the suboptimality gap, and $n_{\bm{a}}(T)$ is the number of pulls. We assume heavy-tailed rewards: there exist $\varepsilon\in(0,1]$ and $v>0$ such that for all $\bm{a}\in\mathcal{A}$,
\begin{equation}\label{eq:moment}
\mathbb{E}[|X_{\bm{a}} - \mu_{\bm{a}}|^{1+\varepsilon}] \le v.
\end{equation}
This allows distributions with infinite variance (when $\varepsilon<1$), capturing Pareto, Student-$t$, and other heavy-tailed families; smaller $\varepsilon$ corresponds to heavier tails.

\subsection{Multi-agent extension}

We extend the setting to $M$ players, where player $i$ has an individual action set $\mathcal{A}_i$ of size $K_i$. The joint action space is $\mathcal{A} = \mathcal{A}_1\times\cdots\times\mathcal{A}_M$, containing $K^M := \prod_{i=1}^M K_i$ joint arms. At each round $t$, each player simultaneously selects an arm, forming joint arm $\bm{a}(t)=(a_1(t),\dots,a_M(t))$, and then observes a reward sampled from $\nu_{\bm{a}(t)}$. The cumulative regret is $R_T = T\mu^\star - \sum_{t=1}^T \mathbb{E}[X_{\bm{a}(t)}]$, where $\mu^\star = \max_{\bm{a}\in\mathcal{A}}\mu_{\bm{a}}$. Players may agree on a strategy beforehand and know each other's action spaces, but \emph{cannot communicate during learning}. We consider three information structures, each matching a distinct class of deployment.

\emph{Problem~A (action asymmetry).} All players observe the same reward realization $X_{\bm{a}(t)}$ but not each other's actions. This is the situation of a team optimizing a single aggregate metric: transmitters in a shared spectrum band that observe total network throughput, or advertising channels evaluated against one conversion count, where the aggregate is instrumented but attribution to individual actions is not.

\emph{Problem~B (reward asymmetry).} Players observe the joint action $\bm{a}(t)$ but each receives an independent sample $X_{\bm{a}(t)}^i \sim \nu_{\bm{a}(t)}$. This matches federated or multi-site experimentation: a configuration is chosen jointly and logged centrally, so every site knows what was deployed, while each site measures only its own privately held outcomes.

\emph{Problem~C (full asymmetry).} Players observe neither others' actions nor a common reward; each receives an i.i.d.\ sample. This models fully decentralized deployments such as sensor or robot teams operating with no backhaul, where each unit sees only its own measurements.

\subsection{Robust upper confidence bounds}\label{sec:rucb}

Throughout, $\widehat{\mu}_{\bm{a}}(t)$ is the truncated mean of~\cite{bubeck2012banditsheavytail}: writing $X_{\bm{a},1},\dots,X_{\bm{a},n_{\bm{a}}(t)}$ for the rewards observed from arm $\bm{a}$,
\begin{equation}\label{eq:trunc}
\widehat{\mu}_{\bm{a}}(t) = \frac{1}{n_{\bm{a}}(t)}\sum_{s=1}^{n_{\bm{a}}(t)} X_{\bm{a},s}\,
\mathbf{1}\!\left\{|X_{\bm{a},s}| \le \left(\tfrac{v s}{\log(T^\gamma)}\right)^{\frac{1}{1+\varepsilon}}\right\}.
\end{equation}
The robust upper confidence bound (RUCB) for joint arm $\bm{a}$ is
\begin{equation}\label{eq:rucb}
\mathrm{RUCB}_{\bm{a}}(t) = \begin{cases}
\infty & \text{if } n_{\bm{a}}(t)=0,\\
\widehat{\mu}_{\bm{a}}(t) + \alpha_{\bm{a}}(t) & \text{otherwise,}
\end{cases}
\end{equation}
where the first case marks an arm from which nothing has yet been observed, so that $\widehat{\mu}_{\bm{a}}(t)$ is undefined; setting the index to $\infty$ forces every joint arm to be played at least once before any comparison is made. The confidence radius is
\begin{equation}\label{eq:alpha}
\alpha_{\bm{a}}(t) = v^{\frac{1}{1+\varepsilon}}\!\left(\frac{c\log(T^\gamma)}{n_{\bm{a}}(t)}\right)^{\!\frac{\varepsilon}{1+\varepsilon}},
\end{equation}
with $c,\gamma>0$, and~\cite[Prop.~1]{bubeck2012banditsheavytail} gives $\Pr(|\widehat{\mu}_{\bm{a}}(t) - \mu_{\bm{a}}| > \alpha_{\bm{a}}(t)) \le t^{-\gamma}$. Only this concentration property is used below, so any estimator obeying a bound of the form~(\ref{eq:alpha})---median-of-means, or the Catoni-style confidence sequences of~\cite{wang2023catoniCS}---may be substituted. Such a substitution changes the constant $c$ and the way $v$ enters, and hence the constants in all three theorems, but not the rates; Catoni-style estimators give the sharpest constants as $\varepsilon\to1$, at the cost of solving an implicit equation at each round.


\section{Problem~A: Common Rewards, Unobserved Actions}\label{sec:probA}

In Problem~A, all players observe the same reward but cannot see others' actions. Two technical challenges arise. First, since actions are hidden, miscoordination may occur if the players' internal estimates diverge, and the observed reward is then attributed to the wrong joint action. Second, the reward distributions are heavy-tailed, requiring robust estimators to control estimation error under weak moment assumptions. However, because rewards are shared, all players' estimates remain identical under the same deterministic update rule---the key simplifying feature. We impose a lexicographic ordering on $\mathcal{A}$ for consistent tie-breaking: $\bm{a} < \bm{b}$ if there exists $n$ such that $a_i=b_i$ for all $i<n$ and $a_n<b_n$. Each player then computes $\mathrm{RUCB}_{\bm{a}}(t)$ for every joint arm and selects the highest, breaking ties lexicographically, yielding \texttt{mRUCB-A} (Algorithm~\ref{algoA}).

\begin{algorithm}[t]
\caption{\texttt{mRUCB-A}}\label{algoA}
\begin{algorithmic}[1]
\STATE Players agree on a lexicographic ordering of $\mathcal{A}$.
\STATE Init.\ $n_{\bm{a}}(0)\gets 0$, $\widehat{\mu}_{\bm{a}}(0)\gets 0$ for all $\bm{a}\in\mathcal{A}$.
\FOR{$t=1,\dots,T$}
    \STATE Compute $\mathrm{RUCB}_{\bm{a}}(t)$ for all $\bm{a}\in\mathcal{A}$ via~(\ref{eq:rucb}).
    \STATE Select $\bm{a}(t)\gets\arg\max_{\bm{a}}\mathrm{RUCB}_{\bm{a}}(t)$ (ties: lexicographic).
    \STATE Pull individual arm; observe common reward; update statistics.
\ENDFOR
\end{algorithmic}
\end{algorithm}

\begin{theorem}\label{thm:probA}
Under condition~(\ref{eq:moment}), if all players follow \texttt{mRUCB-A}, the expected regret satisfies
$R_T = O\!\left(\log(T)\sum_{\bm{a}\in\mathcal{A}}\Delta_{\bm{a}}^{-1/\varepsilon}\right)$.
\end{theorem}

\begin{proof}
Since all players observe the same reward and use the same deterministic update rule with consistent tie-breaking, every player selects the same joint arm at every round. The problem thus reduces to a single-agent heavy-tailed bandit over $K^M$ arms, and the analysis follows~\cite{bubeck2012banditsheavytail}. For each suboptimal arm $\bm{a}$ with gap $\Delta_{\bm{a}}>0$, define the good event at round $t$: $\mathcal{G}_t: |\widehat{\mu}_{\bm{a}}(t)-\mu_{\bm{a}}|\le\alpha_{\bm{a}}(t)$ for all $\bm{a}$. By the concentration bound of Section~\ref{sec:rucb}, $\Pr(\mathcal{G}_t^c)\le K^M t^{-\gamma}$. Under $\mathcal{G}_t$, selection of $\bm{a}$ requires $\widehat{\mu}_{\bm{a}}(t)+\alpha_{\bm{a}}(t)\ge\widehat{\mu}_{\bm{a}^\star}(t)+\alpha_{\bm{a}^\star}(t)$, which implies $2\alpha_{\bm{a}}(t)\ge\Delta_{\bm{a}}$. This fails after $\tau_{\bm{a}} = c\gamma\log(T)(2v^{1/(1+\varepsilon)}/\Delta_{\bm{a}})^{(1+\varepsilon)/\varepsilon}$ pulls. Hence $\mathbb{E}[n_{\bm{a}}(T)] \le \tau_{\bm{a}} + \sum_{t=1}^T\Pr(\mathcal{G}_t^c)$, where the tail sum converges for $\gamma>1$. Summing $\Delta_{\bm{a}}\cdot\mathbb{E}[n_{\bm{a}}(T)]$ over all suboptimal arms gives the result.
\end{proof}

This matches the optimal single-agent heavy-tailed rate over $K^M$ arms. Since the $K^M$ dependence is unavoidable even for a centralized learner, the decentralized agents incur no additional cost from action asymmetry.


\section{Problem~B: Independent Rewards, Observed Actions}\label{sec:probB}

In Problem~B, players observe the joint action but receive \emph{independent} reward samples $X_{\bm{a}(t)}^1,\dots,X_{\bm{a}(t)}^M \stackrel{\text{i.i.d.}}{\sim} \nu_{\bm{a}(t)}$. This reverses Problem~A's structure: players see all actions but their estimates diverge because each empirical mean uses different samples, so one player may conclude that an arm is suboptimal while another player's interval still overlaps. An index rule applied independently by each player would therefore cause persistent miscoordination. \texttt{mRUCB-Intervals} avoids this by replacing index maximization with \emph{round-robin elimination}: players cycle through a common active set $S$, and an arm leaves $S$ only through a signal that every player observes.

For each joint arm $\bm{a}$ and player~$i$ the algorithm maintains
\begin{equation}\label{eq:interval}
I_{\bm{a}}^i(t) = \left[\widehat{\mu}_{\bm{a}}^i(t)-\alpha_{\bm{a}}(t),\;\widehat{\mu}_{\bm{a}}^i(t)+\alpha_{\bm{a}}(t)\right],
\end{equation}
where $\alpha_{\bm{a}}(t)$ is common to all players because it depends only on the shared pull count $n_{\bm{a}}(t)$. Elimination proceeds in three stages. \emph{Detection}: if player $i$ finds that $I_{\bm{a}}^i(t)$ lies strictly below and disjoint from the interval of another active arm, then $\bm{a}$ is dominated from player $i$'s perspective. \emph{Signaling}: player $i$ then deviates from the prescribed action by pulling a \emph{different} individual arm, the only form of implicit communication available. \emph{Propagation}: since actions are observable, all players detect the mismatch between the scheduled joint arm $\bm{a}(t)$ and the realized one $\bm{a}'(t)$, and mark $\bm{a}(t)$ for removal regardless of whether their own intervals support it. Two conventions keep the players' statistics aligned: the reward of a signaling round is discarded, and removals take effect at the end of the current cycle. Algorithm~\ref{algoB} gives the procedure.

\begin{algorithm}[t]
\caption{\texttt{mRUCB-Intervals}}\label{algoB}
\begin{algorithmic}[1]
\STATE Players agree on an ordering of $\mathcal{A}$; set $S\gets\mathcal{A}$, $P\gets\emptyset$.
\WHILE{$t \le T$}
  \FOR{each $\bm{a}\in S$ in order}
    \IF{some player $i$ finds $I_{\bm{a}}^i(t)$ strictly below the interval of another arm of $S$}
      \STATE That player pulls a \emph{different} individual arm; all players observe $\bm{a}'(t)\neq\bm{a}$ and set $P\gets P\cup\{\bm{a}\}$; the reward of this round is discarded.
    \ELSE
      \STATE All players pull the components of $\bm{a}$; player $i$ observes $X_{\bm{a}}^i$ and updates $I_{\bm{a}}^i$ via~(\ref{eq:interval}); $n_{\bm{a}}\gets n_{\bm{a}}+1$.
    \ENDIF
  \ENDFOR
  \STATE $S\gets S\setminus P$; $P\gets\emptyset$.
\ENDWHILE
\end{algorithmic}
\end{algorithm}

\begin{lemma}\label{lem:sync}
Under \texttt{mRUCB-Intervals}, at every round all players hold the same active set $S(t)$ and the same pull counts $n_{\bm{a}}(t)$. Moreover, if $\bm{a}$ is the arm scheduled at round $t$, then $n_{\bm{b}}(t)\ge n_{\bm{a}}(t)$ for every $\bm{b}\in S(t)$.
\end{lemma}

\begin{proof}
Both claims follow by induction on $t$. Initially $S=\mathcal{A}$ and all counts are zero. The scheduled arm is a deterministic function of $S$ and the position in the cycle, which are common by hypothesis. Since actions are observed, every player sees the realized joint arm and applies the same count update, and removals are triggered only by observed deviations, so $S$ remains common. Within a cycle each active arm is scheduled exactly once and signaling rounds increment no counts, so all active arms have equal counts at cycle boundaries and, at any point inside a cycle, the arms not yet scheduled---including the scheduled arm itself---have the smallest counts.
\end{proof}

\begin{theorem}\label{thm:probB}
If all players follow \texttt{mRUCB-Intervals} under~(\ref{eq:moment}), then
\begin{equation}\label{eq:regret_B}
\begin{split}
R_T \le\; & c\gamma\, 4^{\frac{1+\varepsilon}{\varepsilon}}v^{\frac{1}{\varepsilon}}\log(T)\sum_{\bm{a}\neq\bm{a}^\star}\Delta_{\bm{a}}^{-1/\varepsilon}\\
&+ \sum_{\bm{a}\neq\bm{a}^\star}\Delta_{\bm{a}} + (K^M-1)\Delta_{\max} + O(1),
\end{split}
\end{equation}
where the $O(1)$ term collects the contribution of the failure event and is independent of $T$ for $\gamma>1$.
\end{theorem}

\begin{proof}
Let $\mathcal{G}_t$ be the event that $|\widehat{\mu}_{\bm{a}}^i(t)-\mu_{\bm{a}}|\le\alpha_{\bm{a}}(t)$ for all $i$ and all $\bm{a}$; by Section~\ref{sec:rucb} and a union bound over the $MK^M$ player-arm pairs, $\Pr(\mathcal{G}_t^c) \le M K^M t^{-\gamma}$.

\emph{Step~1: $\bm{a}^\star$ is never eliminated.} Under $\mathcal{G}_t$, for every player $i$ and every active $\bm{b}$, the lower end of $I^i_{\bm{a}^\star}$ satisfies $\widehat{\mu}^i_{\bm{a}^\star}+\alpha_{\bm{a}^\star} \ge \mu^\star \ge \mu_{\bm{b}} \ge \widehat{\mu}^i_{\bm{b}}-\alpha_{\bm{b}}$, so $I^i_{\bm{a}^\star}$ never lies strictly below the interval of another active arm. No player signals on $\bm{a}^\star$, and by Lemma~\ref{lem:sync} no player removes it.

\emph{Step~2: Elimination time.} Let $\bm{a}\neq\bm{a}^\star$ be scheduled at round $t$. By Lemma~\ref{lem:sync}, $n_{\bm{a}^\star}(t)\ge n_{\bm{a}}(t)$ and hence $\alpha_{\bm{a}^\star}(t)\le\alpha_{\bm{a}}(t)$. Under $\mathcal{G}_t$ the upper end of $I^i_{\bm{a}}$ is at most $\mu_{\bm{a}}+2\alpha_{\bm{a}}(t)$ and the lower end of $I^i_{\bm{a}^\star}$ is at least $\mu^\star-2\alpha_{\bm{a}}(t)$, so \emph{every} player detects domination as soon as $4\alpha_{\bm{a}}(t)<\Delta_{\bm{a}}$. By~(\ref{eq:alpha}) this holds once $n_{\bm{a}}(t)$ exceeds
\begin{equation}\label{eq:tau_B}
\tau_{\bm{a}} = c\gamma\log(T)\left(\frac{4v^{1/(1+\varepsilon)}}{\Delta_{\bm{a}}}\right)^{(1+\varepsilon)/\varepsilon},
\end{equation}
so that $n_{\bm{a}}(T)\le\tau_{\bm{a}}+1$: the arm is signaled the next time it is scheduled and is removed at the end of that cycle. Note that the detection threshold is $4\alpha_{\bm{a}}$, rather than the $2\alpha_{\bm{a}}$ of an index comparison, because separating two intervals requires both radii to be small.

\emph{Step~3: Signaling cost.} Each arm is removed exactly once, and by Lemma~\ref{lem:sync} its removal consumes exactly one round in which the realized joint action is unintended, contributing regret at most $\Delta_{\max}$. Simultaneous deviations by several players still consume a single round, so the total signaling cost is at most $(K^M-1)\Delta_{\max}$, independent of both $T$ and $M$.

\emph{Step~4: Summing.} By Step~2, $\sum_{\bm{a}\neq\bm{a}^\star}\Delta_{\bm{a}}(\tau_{\bm{a}}+1)$ gives the first two terms of~(\ref{eq:regret_B}). Adding the failure contribution $\sum_{t\le T}\Pr(\mathcal{G}_t^c)\Delta_{\max}\le MK^M\Delta_{\max}\sum_{t\ge1}t^{-\gamma}=O(1)$ for $\gamma>1$, together with Step~3, yields the bound.
\end{proof}

The leading term matches Problem~A up to the constant $4^{(1+\varepsilon)/\varepsilon}$ in place of $2^{(1+\varepsilon)/\varepsilon}$, and the remaining terms are independent of the horizon. The mechanism thus uses action deviations as a 1-bit implicit communication channel; this suffices because $\alpha_{\bm{a}}(t)$ is common across players, so a single detection is enough to eliminate an arm for everyone.


\section{Problem~C: Independent Rewards, Unobserved Actions}\label{sec:probC}

Problem~C combines both asymmetries, eliminating the coordination mechanisms of Problems~A (shared rewards) and~B (observable actions). A player cannot tell whether the realized reward corresponds to the intended joint arm or to a different one caused by another player's deviation, and independent samples simultaneously prevent synchronized estimates. Adaptive, index-driven coordination is therefore unavailable, and exploration must be scheduled deterministically from the round index alone, which every player can reproduce.

\texttt{mHT-DSEE} (Algorithm~\ref{algoC}) follows the DSEE framework~\cite{vakili2013dsee}. Fix an increasing $w(t)\to\infty$ and let $D(t)=\lceil w(t)\log t\rceil$ be the target number of samples of each joint arm by round $t$. If fewer than $K^M D(t)$ exploration rounds have been used, round $t$ is an exploration round and the next joint arm in a fixed cyclic order is played; otherwise every player commits to the maximizer of its own RUCB, computed from exploration samples only. Because the test $N(t)<K^MD(t)$ depends only on $t$, players stay synchronized without communication. Two features matter for the analysis: the schedule is \emph{anytime}, requiring neither $T$ nor the gaps, and the confidence radius uses $\log t$ rather than $\log T$.

\begin{algorithm}[t]
\caption{\texttt{mHT-DSEE}}\label{algoC}
\begin{algorithmic}[1]
\STATE Players agree on an ordering of $\mathcal{A}$ and on $w(t)\uparrow\infty$; $N\gets0$.
\FOR{$t=1,\dots,T$}
    \IF{$N < K^M\lceil w(t)\log t\rceil$}
        \STATE Play the next joint arm in the cyclic order; $N\gets N+1$; each player stores its own reward.
    \ELSE
        \STATE Player $i$ plays its own component of $\arg\max_{\bm{a}}\mathrm{RUCB}^i_{\bm{a}}(t)$, computed from exploration samples.
    \ENDIF
\ENDFOR
\end{algorithmic}
\end{algorithm}

\begin{theorem}\label{thm:probC}
Under~(\ref{eq:moment}), if all players follow \texttt{mHT-DSEE} with $\gamma>1$, then
\begin{equation}\label{eq:regret_C}
R_T \le \Delta_{\max}K^M D(T) + \Delta_{\max}t_0 + O(1),
\end{equation}
where $t_0=\min\{t: w(t) > c\gamma(2v^{1/(1+\varepsilon)}/\Delta_{\min})^{(1+\varepsilon)/\varepsilon}\}$ depends on $\varepsilon,v$ and the gaps but \emph{not} on $T$. With $w(t)=\lceil\log t\rceil$ this gives $R_T = O(K^M\log^2 T)$.
\end{theorem}

\begin{proof}
\emph{Step~1: Exploration.} At most $K^MD(T)$ rounds are exploration rounds, each contributing regret at most $\Delta_{\max}$, which is the first term.

\emph{Step~2: Good event.} At an exploitation round $t$ every arm has $D(t)$ samples for each player. By Section~\ref{sec:rucb} with confidence level $t^{-\gamma}$ and a union bound over the $MK^M$ player-arm pairs, the event $\mathcal{G}_t$ that $|\widehat{\mu}^i_{\bm{a}}-\mu_{\bm{a}}|\le\alpha(D(t))$ for all $i,\bm{a}$ has $\Pr(\mathcal{G}_t^c)\le MK^Mt^{-\gamma}$, where $\alpha(n)=v^{1/(1+\varepsilon)}(c\log(t^\gamma)/n)^{\varepsilon/(1+\varepsilon)}$.

\emph{Step~3: Coordination.} On $\mathcal{G}_t$, $\mathrm{RUCB}^i_{\bm{a}}\le\mu_{\bm{a}}+2\alpha(D(t))$ and $\mathrm{RUCB}^i_{\bm{a}^\star}\ge\mu^\star$ for every player $i$, so every player selects $\bm{a}^\star$ once $2\alpha(D(t))<\Delta_{\min}$, i.e.\ once
\begin{equation}\label{eq:threshold}
D(t) > c\gamma\log(t)\left(\frac{2v^{1/(1+\varepsilon)}}{\Delta_{\min}}\right)^{(1+\varepsilon)/\varepsilon} =: \kappa\log t .
\end{equation}
The threshold does not depend on the player index, so the players agree; and since $D(t)\ge w(t)\log t$, condition~(\ref{eq:threshold}) holds as soon as $w(t)>\kappa$, that is for every $t\ge t_0$. This is the role of the anytime schedule: both sides of~(\ref{eq:threshold}) scale with $\log t$, so the crossing time $t_0$ is determined by $w$ and the gaps alone and does not grow with the horizon. Each player then plays its component of $\bm{a}^\star$, and the realized joint arm is exactly $\bm{a}^\star$.

\emph{Step~4: Exploitation regret.} Exploitation rounds with $t<t_0$ contribute at most $\Delta_{\max}t_0$. For $t\ge t_0$, regret is incurred only on $\mathcal{G}_t^c$, contributing at most $MK^M\Delta_{\max}\sum_{t\ge1}t^{-\gamma}=O(1)$ for $\gamma>1$. Adding Step~1 gives~(\ref{eq:regret_C}). Taking $w(t)=\lceil\log t\rceil$ gives $D(T)=O(\log^2T)$.
\end{proof}

\begin{remark}
The choice of $w$ trades the two terms of~(\ref{eq:regret_C}) against each other: faster growth shortens $t_0$ but enlarges the exploration budget $D(T)$. Any $w\uparrow\infty$ yields $o(\log^{1+\eta}T)$ regret for the corresponding $\eta$, and knowledge of $\Delta_{\min}$ would allow the constant schedule $w\equiv\kappa$ and hence $O(\log T)$; the extra factor is the price of not knowing the gaps.
\end{remark}


\section{Experiments}\label{sec:experiments}

\subsection{Setup}

We take $M=2$ players, $K=2$ individual arms ($K^M=4$ joint arms), and horizon $T=10^6$, averaging over $10$ independent runs on the fixed instance $\bm{\mu}=(0.44,0.57,0.91,0.25)$, so that the gaps are $(0.46,0.34,-,0.65)$. Pulling joint arm $a$ yields a Pareto reward with shape $a_0=2$ and scale $x_m=\mu_a(a_0-1)/a_0$, so that $\mathbb{E}[X_a]=\mu_a$. This distribution has finite mean and \emph{infinite variance}: its centered moments of order $1+\varepsilon$ are finite exactly for $\varepsilon<1$. We therefore set $\varepsilon=0.5$, for which $\mathbb{E}|X_a-\mu_a|^{1.5}\le0.75<1=v$ for every $\mu_a\in[0,1]$, so~(\ref{eq:moment}) holds; taking $\varepsilon=1$ would instead require a finite second moment and is not admissible for this reward family. All three algorithms use the truncated-mean estimator~(\ref{eq:trunc}) with $(c,\gamma)=(1,2)$, i.e.\ exactly the estimator the analysis assumes, and \texttt{mHT-DSEE} uses $w(t)=\lceil\log t\rceil$.

\subsection{Results}

\begin{figure}[t]
    \centering
    \includegraphics[width=0.9\linewidth]{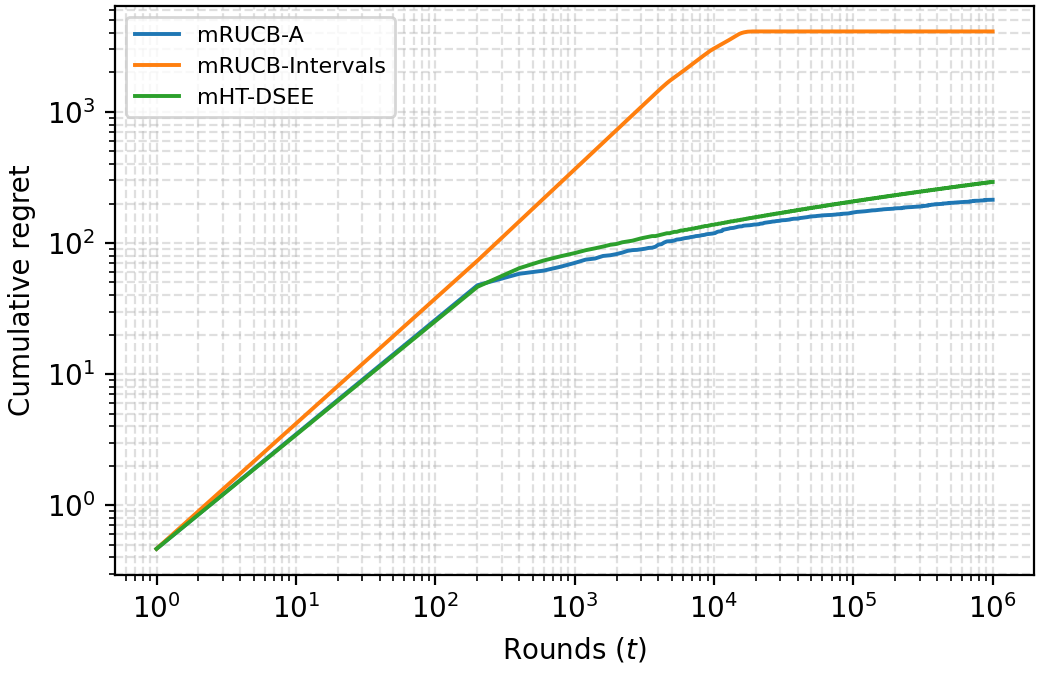}
    \caption{Mean cumulative regret over $10$ runs under Pareto rewards with infinite variance ($M{=}2$, $K{=}2$, $T{=}10^6$), log--log axes.}
    \label{fig:regret}
\end{figure}

Figure~\ref{fig:regret} shows mean cumulative regret. All three curves are clearly sublinear, confirming that each algorithm identifies $\bm{a}^\star$ under infinite-variance rewards and with the robust estimator in place. \texttt{mRUCB-A} ends at $214\pm24$ and \texttt{mHT-DSEE} at $292\pm21$, both still growing slowly, while \texttt{mRUCB-Intervals} ends at $4115\pm285$ but is exactly flat beyond ${\approx}7\times10^4$ rounds.

The ordering at this horizon is governed by constants rather than by rates, and is instructive. Elimination in Problem~B requires two intervals to separate, i.e.\ $4\alpha_{\bm{a}}<\Delta_{\bm{a}}$, whereas the index comparisons of Problems~A and~C need only $2\alpha_{\bm{a}}<\Delta_{\bm{a}}$; by~(\ref{eq:alpha}) this is $2^{(1+\varepsilon)/\varepsilon}=8$ times more samples of each arm when $\varepsilon=0.5$, which is what the early growth of the orange curve buys. The payoff is that once the active set collapses, \texttt{mRUCB-Intervals} incurs no further regret at all, whereas the other two keep exploring; the curves would therefore cross at larger horizons. Similarly, \texttt{mHT-DSEE} is inexpensive here because its exploration budget $K^M\lceil\log^2 t\rceil$ is only ${\approx}760$ rounds at $T=10^6$, even though its rate is the worst of the three. The experiments thus support the theory while showing that the hierarchy of Table~\ref{tab:summary} is asymptotic: at moderate horizons the constants attached to each coordination mechanism dominate. A further practical caveat is that the joint action space grows as $K^M$, so larger instances lengthen the exploration phases of \texttt{mHT-DSEE} and slow the elimination cascade of \texttt{mRUCB-Intervals}.


\section{Conclusion}

We studied multi-agent bandits with heavy-tailed rewards under three information asymmetries. Our algorithms show that effective decentralized learning is achievable even under significant asymmetry and non-sub-Gaussian noise: shared rewards (Problem~A) enable costless synchronization; observed actions (Problem~B) provide an implicit signaling channel whose cost is independent of the horizon and of the number of players; and the fully asymmetric setting (Problem~C) requires a pre-committed anytime schedule at a $\log T$ factor of additional cost. That observable actions compensate for the loss of shared rewards at leading order is a notable positive result, while the barrier between Problems~B and~C shows the value of even minimal observability.

Two limitations point to future work. First, our guarantees, like those of~\cite{bubeck2012banditsheavytail, vakili2013dsee}, assume that $(\varepsilon,v)$ are known: a conservative choice (smaller $\varepsilon$ or larger $v$) keeps every bound valid but inflates the confidence radius and hence the regret, so adapting to unknown tail heaviness---plausibly through self-normalized constructions such as~\cite{wang2023catoniCS}---remains open. Second, tighter lower bounds for Problem~C would clarify whether the extra $\log T$ factor is necessary without shared information. Extensions to adversarial or non-stationary rewards, and structured reward models such as linear or factored bandits that would mitigate the exponential $K^M$ dependence, are also natural directions.


\bibliographystyle{ieeetr}
\bibliography{references}

\end{document}